\documentclass{article}
\usepackage{ijcai24}
\usepackage{times}
\usepackage{soul}
\usepackage{url}
\usepackage[utf8]{inputenc}
\usepackage[small]{caption}
\usepackage{graphicx}
\usepackage{amsmath}
\usepackage{amssymb}
\usepackage{amsthm}
\usepackage{booktabs}
\usepackage{algorithm}
\usepackage{algorithmic}
\usepackage[switch]{lineno}
\usepackage{multirow}
\usepackage{multicol}
\usepackage{mathtools}
\usepackage{bm}
\usepackage[hidelinks]{hyperref}
\usepackage{cleveref}
\usepackage{makecell}
\usepackage{xcolor}
\usepackage{enumitem}

\newtheorem{theorem}{Theorem}[section]

\newtheorem{lemma}{Lemma}[section]
\newtheorem{assumption}{Assumption}[section]

\title{Domain Adaptive and Fine-grained Anomaly Detection for Single-cell Sequencing Data and Beyond}

\author{
Kaichen Xu$^1$\thanks{Equal contribution.}
\and
Yueyang Ding$^1$$^{\ast}$\and
Suyang Hou$^2$\and
Weiqiang Zhan$^1$\and
Nisang Chen$^1$ \and
Jun Wang$^3$ \and 
Xiaobo Sun$^1$\thanks{Corresponding author.}
\affiliations
$^1$School of Statistics and Mathematics, Zhongnan University of Economics and Law \\
$^2$School of Information Engineering, Zhongnan University of Economics and Law \\
$^3$iWudao Tech \\
\emails
\{kaichenxu, yueyangding, suyang, weiqiangzhan, nisangchen\}@stu.zuel.edu.cn, xsun28@gmail.com, jwang@iwudao.tech
}

\begin{document}

\maketitle
\begin{abstract}
Fined-grained anomalous cell detection from affected tissues is critical for clinical diagnosis and pathological research. Single-cell sequencing data provide unprecedented opportunities for this task. However, current anomaly detection methods struggle to handle domain shifts prevalent in multi-sample and multi-domain single-cell sequencing data, leading to suboptimal performance. Moreover, these methods fall short of distinguishing anomalous cells into pathologically distinct subtypes. In response, we propose ACSleuth, a novel, reconstruction deviation-guided generative framework that integrates the detection, domain adaptation, and fine-grained annotating of anomalous cells into a methodologically cohesive workflow. Notably, we present the first theoretical analysis of using reconstruction deviations output by generative models for anomaly detection in lieu of domain shifts. This analysis informs us to develop a novel and superior maximum mean discrepancy-based anomaly scorer in ACSleuth. Extensive benchmarks over various single-cell data and other types of tabular data demonstrate ACSleuth's superiority over the state-of-the-art methods in identifying and subtyping anomalies in multi-sample and multi-domain contexts. Our code is available at \href{https://github.com/Catchxu/ACsleuth}{https://github.com/Catchxu/ACsleuth}.
\end{abstract}

\section{Introduction}
The detection and differentiation of anomalous cells (ACs) from affected tissues, which we refer to as Fine-grained Anomalous Cell Detection (FACD), is critical for investigating the pathological heterogeneity of diseases, significantly contributing to the clinical diagnostics, biomedical research, and the development of targeted therapies \cite{Meaning}. Single-cell (SC) sequencing data, referred to as SC data for brevity, such as single-cell RNA-sequencing (scRNA-seq) and single-cell ATAC-sequencing (scATAC-seq) data, provide unprecedented opportunities for FACD analysis from distinct perspectives \cite{scGAD,flow,SampleQC}. For example, a typical scRNA-seq dataset is organized as a tabular matrix $\mathbf{X}\in \mathbb{R}^{N\times G}$, where $\mathbf{X}_{i,j}$ represents the expression read counts of the $j$-th gene in the $i$-th cell. This dataset characterizes gene expression levels within single cells, thus facilitating the FACD.
\begin{figure}[t]
    \centering
    \includegraphics[width=\linewidth]
    {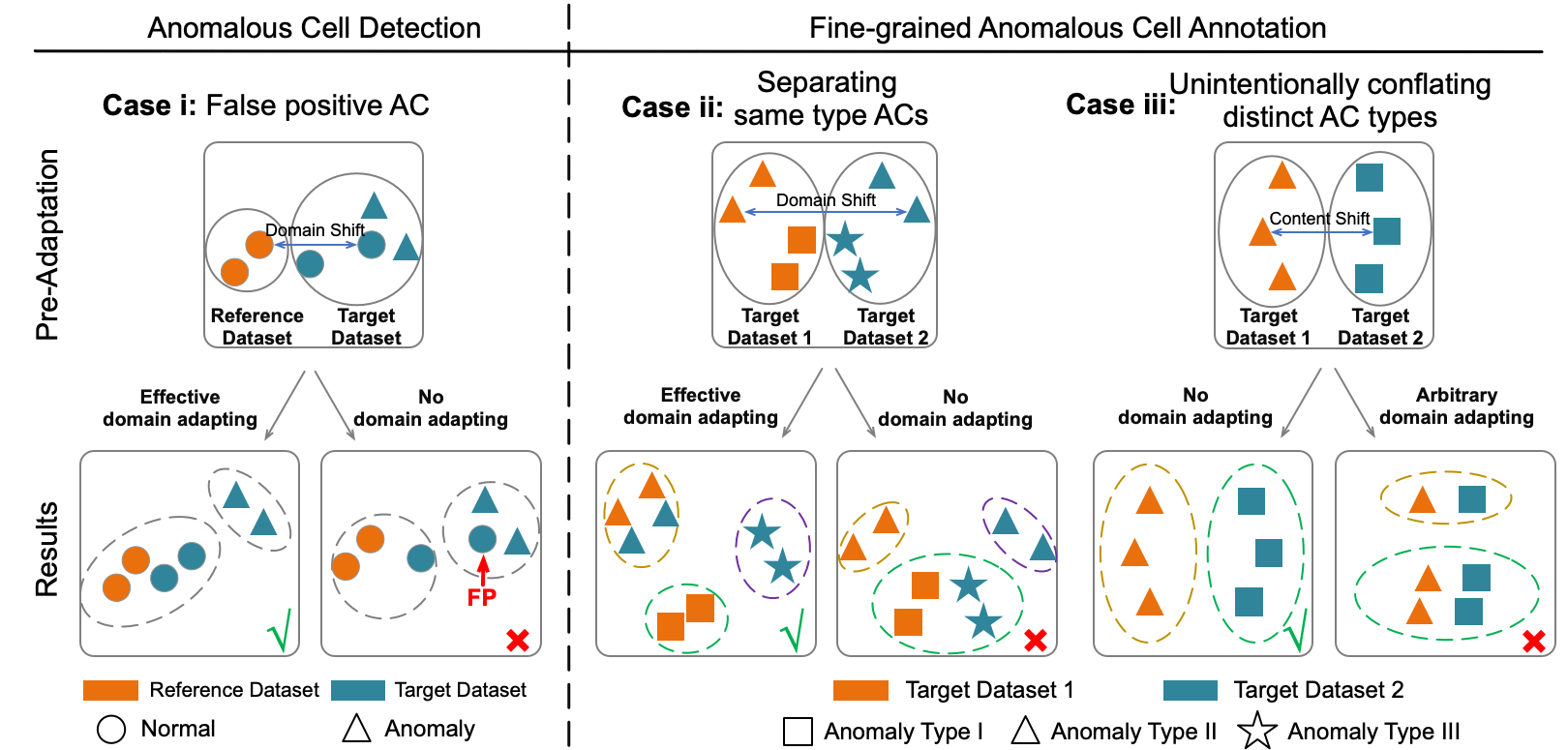}
    \caption{Three error types in multi-sample and multi-domain FACD analysis due to domain shift and sample-specific AC types. Domain shift represents non-biological variations caused by technical differences among samples, while content shift represents true biological variations.}
    \label{fig:illustration}
\end{figure}

The FACD workflow is a two-step process: the detection of anomalous cells as a whole followed by their fine-grained annotating. The detection step aims to identify ACs in target datasets by comparing them to “normality” defined by a reference dataset containing normal cells only. The subsequent fine-grained annotating step involves clustering the detected ACs into different groups, each representing a biologically distinct subtype. However, de novo FACD using SC data presents a significant challenge, especially in multi-sample contexts, primarily due to various types of cross-sample \textbf{Domain Shifts} (DS), i.e., batch effects, stemming from non-biological variations in data types, sequencing technologies, and experimental conditions \cite{batchco,DoS}. These DS often intermingle with content shifts, which originate from biological variations and represent true signals crucial for AC detection, leading to several types of errors in FACD analysis as shown in \Cref{fig:illustration}: \textbf{i}) \textbf{False positive AC}. Normal cells in the target dataset can be erroneously labeled as anomalous due to the DS relative to the reference dataset; \textbf{ii}) \textbf{Separating same type ACs}. ACs of the same type across different target datasets may appear very different due to DS, leading to their misclassification as distinct types; \textbf{iii}) \textbf{Unintentionally conflating distinct AC types}. Efforts to mitigate DS can inadvertently diminish content shifts critical for distinguishing between sample-specific AC types, leading to their merging as one type. 


Anomaly detection (AD) methods generally involve assigning an anomaly score to each target instance, followed by setting a threshold for anomaly determination as per the user's tolerance for false positives \cite{scorethre}. This strategy is followed by approaches tailored for AC detection, including uncertainty-based and generative methods \cite{CAMLU}. The former methods, exemplified by scmap \cite{scmap}, assign anomaly scores based on cluster assignment uncertainty, while the latter methods, such as CAMLU \cite{CAMLU}, attempt to reconstruct cells and output reconstruction deviations as anomaly scores. These methods either are susceptible to DS or fail to provide fine-grained annotation, i.e., labeling all ACs as ``unassigned'' \cite{petegrosso2020machine,scPOT}. To mitigate DS, domain adaptation methods like MNN \cite{MNN} are often used before AC detection, yet with limited effectiveness owing to their sensitivity to ACs \cite{MNNno1}. For fine-grained AC annotation, researchers often resort to clustering methods, 
which are not specifically oriented towards this purpose and thus struggle to capture subtle differences among ACs. Particularly, this separation of fine-grained AC annotation as an independent clustering task can disrupt the methodological coherence of the workflow, leading to the overall underperformance of FACD \cite{chen2020integrating}. MARS \cite{MARS} and scGAD (also known as scPOT) \cite{scGAD,scPOT} represent the only two available methods that integrate AC detection and fine-grained annotation. Yet, MARS, lacking a domain adaptation mechanism, exhibits compromised performance in lieu of DS \cite{scPOT}. scGAD falls short of handling multiple target datasets due to its reliance on MNN for domain adaptation, which is effective primarily in single target dataset scenarios \cite{MNNno3}. 

Inspired by the versatility of generative adversarial networks (GAN) in AD \cite{ALAD} and domain adaptation-related style-transfer \cite{cycleGAN} tasks, we propose ACSleuth, an innovative GAN-based framework designed to integrate AC detection, domain adaptation, and fine-grained annotation into a cohesive workflow, as shown in \Cref{fig:workflow}. For AC detection (\textbf{Phase I}), a specialized GAN module (\textit{module I}) is trained on the reference dataset and then applied to target datasets to discriminate between normal and anomalous cells based on their anomaly scores. \textit{Module I} differs from existing GAN-based AD methods in incorporating a memory block to reduce mode collapse risk \cite{gong2019memorizing} and featuring a novel maximum mean discrepancy (MMD)-based anomaly scorer. This scorer translates cell reconstruction deviations into anomaly scores, whose efficacy in facilitating AD is theoretically proved in \textbf{Theorem 3.3}. In the subsequent \textbf{Phase II}, ACSleuth utilizes a second GAN module (\textit{module II}) to pinpoint ``kin'' normal cell pairs between the reference and target datasets, from which DS inherent in each target dataset can be learned as a matrix for domain adaptation. This approach prevails over existing domain adaptation methods for SC data, e.g., MNN, in preserving the data's original scale and semantic integrity, simultaneously adapting multiple datasets, and being resilient against complications caused by sample-specific AC types. Finally, ACSleuth performs a self-paced deep clustering (\textbf{Phase III}) on fusions of embeddings and reconstruction deviations of domain-adapted ACs, obtaining iteratively enhanced fine-grained annotations. Overall, our main contributions are:
\begin{itemize}[left=0pt]
\setlength{\itemsep}{0.01pt}
    \item We propose an innovative framework for FACD in SC data, which integrates AC detection, domain adaptation, and fine-grained annotation into a cohesive workflow. 
    \item We provide the first theoretical analysis of using reconstruction deviations for AD in lieu of DS (\textbf{Theorem 3.1-3.3}). This analysis also informs us to develop a novel and effective MMD-based anomaly scorer.
    \item  We introduce a novel domain adaptation method in multi-sample and multi-domain contexts. 
    \item  We innovatively fuse reconstruction deviations with domain-adapted cell embeddings as clustering inputs for enhanced AC fine-grained annotations.
     \item  ACSleuth has been rigorously benchmarked in various experimental scenarios regarding DS types, target sample quantity, and sample-specific AC types. ACSleuth consistently outperforms the state-of-the-art (SOTA) methods across all scenarios. Furthermore, ACSleuth is also applicable to other types of tabular data.
\end{itemize}

\section{Related Works}

\subsection{Anomalous Single Cell Detection}
AD methods tailored for SC data, such as scmap, CHETAH \cite{de2019chetah}, and scPred \cite{alquicira2019scpred}, typically involve cell classification or clustering, labeling cells with uncertain assignments as ACs. These methods suffer from high false positive rates, as the assignment uncertainty might stem from similarities among normal cell types rather than the presence of ACs \cite{CAMLU}. To circumvent this issue, CAMLU directly discriminates ACs from normal cells using informative genes selected based on their reconstruction deviations. Moreover, AD methods for general types of tabular data, categorized as generative, one-class classification, and contrastive approaches, are also adaptable to SC data. Generative methods like RCA \cite{RCA} and ALAD \cite{ALAD} detect anomalies based on instance reconstruction deviations, using collaborative autoencoders (AE) and a GAN model, respectively. DeepSVDD \cite{DeepSVDD18}, a one-class classification method, learns a finite-sized ``normality'' hypersphere in a latent space and labels outside instances as anomalies. Contrastive methods ICL \cite{ICL} and NeuTraL \cite{NeuTraL} generate different views for contrastive learning using mutual information-based mappings and learnable neural transformations, respectively, with anomalies identified as those with large contrastive losses. Lastly, a recent study SLAD \cite{SLAD} introduces scale learning to identify anomalies as those unfitted to the scale distribution of inlier representations. A common limitation across these methods is their inability to account for DS between reference and target datasets and to further categorize ACs into fine-grained subtypes.

\subsection{Fine-grained Anomalous Cell Annotation}
As a task following AC detection, fine-grained AC annotation in SC data is often approached as an independent clustering task using general-purpose clustering methods like Leiden \cite{leiden}, EDESC \cite{EDESC}, and DFCN \cite{DFCN}, or methods specifically designed for SC data, such as scTAG \cite{scTAG}, Seurat V5 \cite{Seurat}, and scDEFR \cite{cheng2023unsupervised}. EDESC, DFCN, scTAG, and scDEFR, which are based on DEC \cite{xie2016unsupervised}, adopt a schema of joint optimization of instance representation learning and discriminatively boosted clustering, yet differ in their specific approaches for learning representations and calculating cluster soft assignments. Seurat V5, on the other hand, first builds a KNN graph based on cell-cell similarities and then applies community detection clustering using the Louvain algorithm. A significant challenge with these methods is their dependency on the accuracy of the AC detection in the first place. Treating AC detection and fine-grained annotation as separate tasks can disrupt the methodological coherence of the FACD workflow, potentially compromising its overall performance. To tackle this issue, scGAD and MARS integrate both tasks into a unified process. MARS, for example, learns a latent space where cells positioned near landmarks of distinct anomaly subtypes are assigned corresponding anomaly annotations. While these methods have shown promising results when all ACs are from the same domain, their efficacy can be undermined in the presence of DS. To tackle this, scGAD employs the MNN to adapt DS before fine-grained annotating. However, MNN's limitation to a single reference and target dataset makes scGAD less suitable for multi-sample and multi-domain contexts. 


\begin{figure}[t]
    \centering
    \includegraphics[width=0.85\linewidth]
    {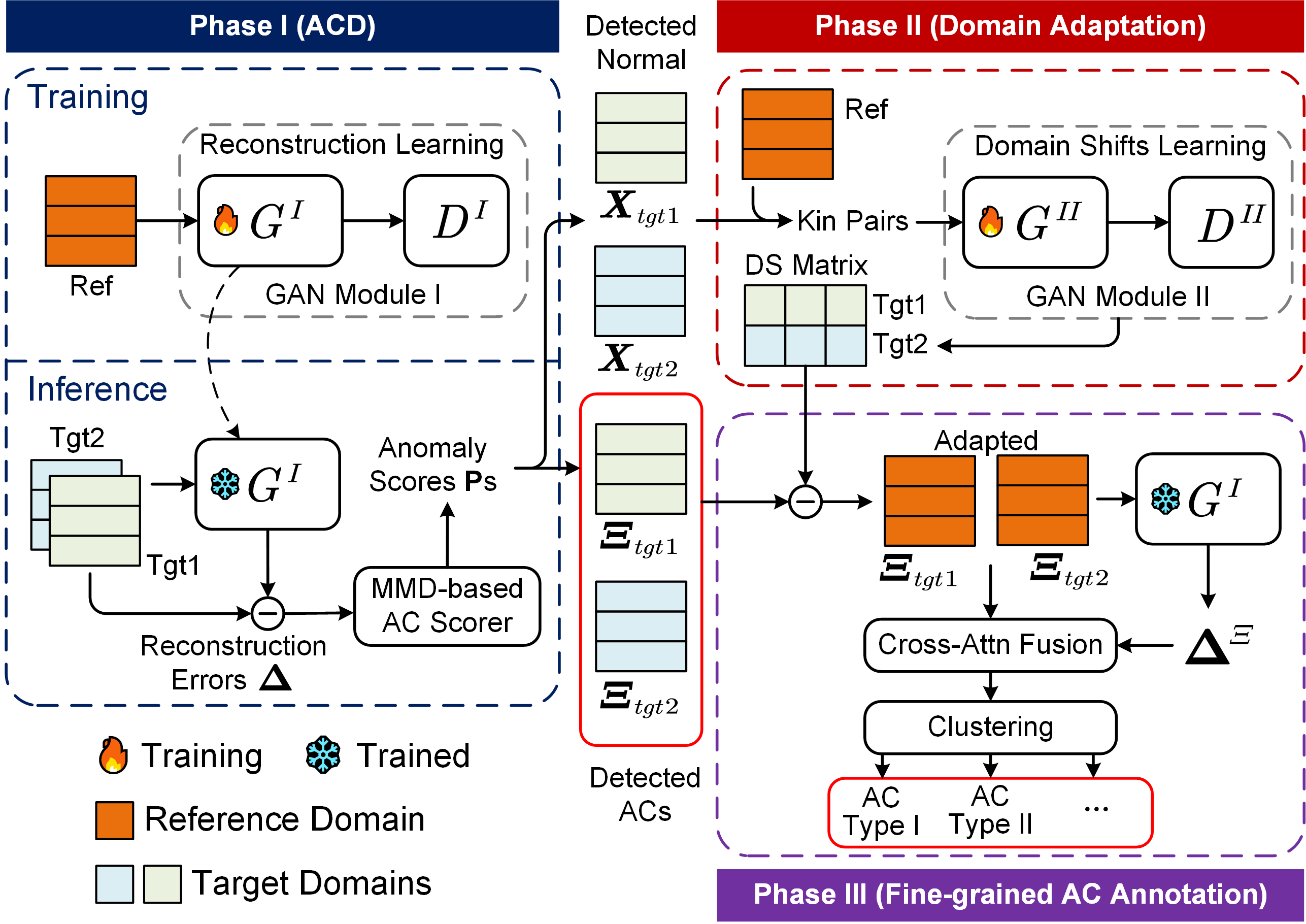}
    \caption{The workflow of ACSleuth.}
    \label{fig:workflow}
\end{figure}
\vspace{-0.2cm}

\section{Methods}
ACSleuth's workflow (\Cref{fig:workflow}) comprises three phases corresponding to AC detection (\textbf{Phase I}), domain adaptation (\textbf{Phase II}), and fine-grained annotation (\textbf{Phase III}).

\subsection{Anomalous Cell Detection (Phase I)}
In Phase I, a GAN model (\textit{module I}) is initially trained to reconstruct normal cells in the reference dataset, and then applied to target datasets. Anomaly scores, derived from reconstruction deviations, are then used to identify ACs. Let $\bm{x}_i \in \mathbb{R}^{N_{gene}}$ denote the gene expression vector of a cell $i\in S^k$, $S^k$ denote the single-cell set in domain $k$, $N_{gene}$ the number of genes. Then we have the following assumption \cite{hornung2016combining}:
\begin{assumption}\label{as:pattern}
    \begin{equation}\label{eq:pattern}
        \bm{x}_i = \bm{x}_i^* + \bm{b}^k + \bm{\epsilon}_i,
    \end{equation}
    where $\bm{x}_i^*$ denotes cell $i$'s biological content, $\bm{b}^k$ the non-biological content specific to domain $k$, $\bm{\epsilon}_i\sim N(0, \bm{\sigma}_i^2)$ a random Gaussian noise. 
\end{assumption}

\paragraph{GAN Model Training.} 
\textit{Module I} comprises a generator $G^I$ and a discriminator $D^I$. The generator itself is composed of three components: an MLP-based encoder ($G^I_E:\mathbb{R}^{N_{gene}} \rightarrow \mathbb{R}^p$), an MLP-based decoder ($G^I_D:\mathbb{R}^p \rightarrow \mathbb{R}^{N_{gene}}$), and a queue-based memory block ($\bm{Q}\in\mathbb{R}^{N_{mem}\times p}$ ), where $p$ is the cell embedding dimension, $N_{mem}$ is the memory block size. The embedding of any cell $i$ can be expressed as:
\begin{equation}
    \bm{z}_i = G^I_E(\bm{x}_i)\in \mathbb{R}^p.
\end{equation}
The memory block provides an attention-based means to reconstruct $\bm{z}_i$ as $\widetilde{\bm{z}}_i$:
\begin{equation}
    \widetilde{\bm{z}}_i = \bm{Q}^T \text{softmax}\left(\frac{\bm{Q}\bm{z}_i}{\tau}\right)\in \mathbb{R}^p,
\end{equation}
where $\tau$ is a temperature hyperparameter. During the training, $\bm{Q}$ is dynamically updated by enqueuing the most recent $\widetilde{\bm{z}}$ and dequeuing the oldest ones to maintain a balance between preserving learned features and adapting to new data, thus reducing the risk of mode collapse. $G^I_D$ reconstructs cell $i$ from $\widetilde{\bm{z}}_i$ as $\widehat{\bm{x}}_i$:
\begin{equation}
    \widehat{\bm{x}}_i = G^I_D(\widetilde{\bm{z}}_i)\in \mathbb{R}^{N_{gene}}.
\end{equation}
The discriminator $D^I$ comprises an encoder, similar to $G^I_E$, and an MLP-based classifier. $D^I$ is trained to distinguish between $\bm{x}$ and $\widehat{\bm{x}}$ . The loss functions for the generator ($\mathcal{L}_{G^I}$) and discriminator ($\mathcal{L}_{D^I}$) are defined as:
\begin{equation}
    \mathcal{L}_{G^I} = \alpha\mathcal{L}_{\text{rec}} + \beta\mathcal{L}_{\text{adv}} 
    = \alpha\mathbb{E}\left[\Vert \bm{x} - \widehat{\bm{x}}\Vert_1\right] - \beta\mathbb{E}\left[D^I(\widehat{\bm{x}})\right], 
\end{equation}
\begin{equation}
\begin{aligned}
    \mathcal{L}_{D^I} = \mathbb{E}\left[D^I(\widehat{\bm{x}})-D^I(\bm{x})\right] + \lambda\mathbb{E}\left[\left({\Vert \nabla_{\widetilde{\bm{x}}} D^I(\widetilde{\bm{x}}) \Vert}_2 - 1\right)^2\right],
\end{aligned}
\end{equation}
where $\widetilde{\bm{x}}=\epsilon\widehat{\bm{x}}+(1-\epsilon){\bm{x}},\ \epsilon\in(0,1)$. Here, $\mathcal{L}_{\text{rec}}$ represents the reconstruction loss, $\mathcal{L}_{\text{adv}}$ the adversarial loss, and $\alpha, \beta, \lambda\geq0$ the weights of three loss terms. A gradient penalty term applied to $\widetilde{\bm{x}}$ ensures the Lipschitz continuity of the discriminator \cite{wGANs}.

\paragraph{MMD-based anomalous cell scorer.}
Let $\bm{x}_i$ denote a normal cell, and $\bm{\xi}_j$ an anomalous cell, $\bm{\delta}_i^x \coloneqq \bm{x}_i - \widehat{\bm{x}}_i$, $\bm{\delta}_j^\xi \coloneqq \bm{\xi}_j - \widehat{\bm{\xi}}_j$ denote their reconstruction deviations. The scorer aims to utilize the MMD metric to translate $\bm{\delta}$s into anomaly scores that can represent the maximized discrepancy between $\bm{\delta}^x$ and $\bm{\delta}^\xi$ in the Reproduced Hilbert Kernel Space (RHKS), thus more effective facilitating the differentiation between normal and anomalous cells. Specifically, given the distributions $\bm{\delta}_i^x \sim p, \bm{\delta}_j^\xi \sim q$, the MMD metric is calculated as \cite{gretton2012kernel}:
\begin{equation}\label{eq:MMD}
    \begin{aligned}
        & MMD^2\left(\bm{\delta}_i^x, \bm{\delta}_j^\xi\right) =  \Vert \frac{1}{m}\sum_{i=1}^{m}\phi(\bm{\delta}_i^x) -  \frac{1}{n}\sum_{j=1}^{n}\phi(\bm{\delta}_j^\xi)\Vert_{\mathcal{H}}^2 \\
        & =  \mathbb{E}_{\bm{\delta}^x, (\bm{\delta}^x)^{\prime} \sim p}\left[k(\bm{\delta}^x, (\bm{\delta}^x)^{\prime})\right] +
        \mathbb{E}_{\bm{\delta}^\xi, (\bm{\delta}^\xi)^{\prime} \sim q}\left[k(\bm{\delta}^\xi, (\bm{\delta}^\xi)^{\prime})\right] \\
        & - 2\mathbb{E}_{\bm{\delta}^x \sim p, \bm{\delta}^\xi \sim q}\left[k(\bm{\delta}^x, \bm{\delta}^\xi)\right],
    \end{aligned}
\end{equation}
where $k$ is a positive definite kernel function, $m$ and $n$ denote the actual yet unknown numbers of inliers and anomalies in the target dataset. An AC scorer function can be trained by minimizing the loss function:
\begin{equation}\label{eq:def}
        \min_{\bm{\delta}_m^x, \bm{\delta}_n^\xi} \quad \mathcal{L}_p\left(\bm{\delta}_m^x, \bm{\delta}_n^\xi\right) =  -MMD^2\left(\bm{\delta}_m^x, \bm{\delta}_n^\xi\right),
\end{equation}
where $\bm{\delta}^{x}_m=\{\bm{\delta}^{x}_i | \bm{\delta}^{x}_i\in \mathbb{R}^{N_{gene}},i=1,2,\cdots,m]\}$, $\bm{\delta}^{\xi}_n=\{\bm{\delta}^{\xi}_j | \bm{\delta}^{\xi}_j\in \mathbb{R}^{N_{gene}},j=1,2,\cdots,n\}$ denote the sets of reconstruction deviations of inliers and anomalies, respectively.

As shown in the theoretical analysis in section 3.4, given a linear kernel $k$, the loss function in equation  \eqref{eq:def} can be reformulated as:
\begin{equation}\label{eq:def_p}
    \mathcal{L}_p \left(\bm{\delta}_m^x, \bm{\delta}_n^\xi\right)=   
   -\sum_{i}^{m+n}\sum_{j\neq i}^{m+n} k(\bm{\delta}_i, \bm{\delta}_j) \gamma_c(p_i, p_j),
\end{equation}
where $p_i \coloneqq f_p(\bm{\delta}_i)\in (0, 1)$ is instance $i$'s anomaly score, and $f_p$ is a trainable MLP-based function. $\gamma_c(p_i, p_j): (0, 1) \times (0, 1) \rightarrow \mathbb{R}$ defines a mapping function as:
\begin{equation}\label{eq:gamma_c}
    \begin{aligned}
        \gamma_c(p_i, p_j) \coloneqq & \frac{\sin\pi p_i \sin\pi p_j}{\pi^2}\left(\frac{[m(m-1)]^{-1}}{p_ip_j} - \frac{(mn)^{-1}}{(p_i-1)p_j} \right.\\
        & \left.- \frac{(mn)^{-1}}{p_i(p_j-1)} + \frac{[n(n-1)]^{-1}}{(p_i-1)(p_j-1)}\right).
    \end{aligned}
\end{equation}
Here, $\widetilde{n} = \sum_{i}^{m+n} p_i$ and $\widetilde{m} = (m+n) - \widetilde{n}$ are used in place of the unknown $m$ and $n$. During the training, $p,\widetilde{n}$ and $\widetilde{m}$ are iteratively updated until convergence. 
Upon training completion, anomalies can be identified based on their anomaly scores $p$. Particularly, the property of $\mathcal{L}_p$'s bounded variation across domains endows our AC scorer's robustness against DS, as formalized in section 3.4.

\subsection{Multi-sample Domain adaptation (Phase II)}
This phase aims to adapt non-biological DS among ACs to expose their biological differences for fine-grained cell annotating. Initially, ACs detected in Phase I are excluded from the datasets to eliminate their interference. The DS of $N_{target}$ datasets relative to the reference dataset are then learned as a matrix $\bm{S} \in \mathbb{R}^{N_{target} \times p}$ through the generator ($G^{II}$) of the second GAN module (\textit{module II}), which has an encoder ($G^{II}_E$) and a decoder ($G^{II}_D$) similar to those in \textit{module I}.  A target cell $\bm{x}$ is adapted to the reference domain as follows:
\begin{equation}
    \begin{gathered}
        \bm{z} = G^{II}_E(\bm{x})\in \mathbb{R}^p, \quad \widetilde{\bm{z}} = \bm{z} - \bm{S}^T \bm{b}, \\ \widehat{\bm{x}} = G^{II}_D(\widetilde{\bm{z}}), \\
    \end{gathered}
\end{equation}
where $\bm{b}\in \{0, 1\}^{N_{target}}$ denotes the cell's one-hot encoded domain-identity vector, $\widehat{\bm{x}}$ represents the cell post-domain adaptation. The discriminator of \textit{module II}, $D^{II}$, is trained to discriminate $\widehat{\bm{x}}$ as either from the reference domain or not. The loss functions of $G^{II}$ and $D^{II}$ are:
\begin{equation}
    \mathcal{L}_{G^{II}} = \alpha\mathbb{E}\left[\Vert \bm{x}^+ - \widehat{\bm{x}}\Vert_1\right] - \beta\mathbb{E}\left[D^{II}(\widehat{\bm{x}})\right],
\end{equation}
\begin{equation}
\begin{aligned}
    \mathcal{L}_{D^{II}} & = \mathbb{E}\left[D^{II}(\widehat{\bm{x}})\right] - \mathbb{E}\left[D^{II}(\bm{x}^+)\right] + \\ & \qquad
    \lambda\mathbb{E}\left[\left({\Vert \nabla_{\widetilde{\bm{x}}}D^{II}(\widetilde{\bm{x}}) \Vert}_2 - 1\right)^2\right], 
\end{aligned}
\end{equation}
where $\widetilde{\bm{x}}=\epsilon\widehat{\bm{x}}+(1-\epsilon){\bm{x}^+},\ \epsilon\in(0,1)$. $\bm{x}^+$ is a ``kin'' cell from the reference dataset that is most similar to $\bm{x}$ in their biological contents. For each target cell $\bm{x}_i$, its ``kin'' reference cell $\bm{x}_i^+$ is identified as:
\begin{equation}
    \bm{x}_i^+ = \mathop{\arg\max}\limits_{\bm{x}_j^+ \in \mathcal{X}_{ref}} K\left(G^{II}_E(\bm{x}_i), G^{II}_E(\bm{x}_j^+)\right),
\end{equation}
where $\mathcal{X}_{ref}$ denotes the reference dataset, $K$ is a Gaussian kernel function as in \cite{GraphRNN}. 

\subsection{Fine-grained Anomalous Cell Annotation (Phase III)}
In Phase III,  we first obtain \textit{module I}-generated reconstruction deviations (denoted as $\bm{\Delta}\in \mathbb{R}^{N_{an} \times N_{gene}}$)  of ACs post-domain adaptation (denoted as $\bm{\Xi} \in \mathbb{R}^{N_{an} \times N_{gene}}$), where $N_{an}$ is the number of detected ACs:
\begin{equation}
{\widehat{\bm{\Xi}}}=G^I_D\left(G^I_E(\bm{\Xi})\right), \quad \bm{\Delta} = \bm{\Xi}-\widehat{\bm{\Xi}}.
\end{equation}

Next, AC embeddings and reconstruction deviations are fused using a cross-attention fusion block:
\begin{equation}
    \begin{gathered}
        \bm{\Psi} = \text{Multi-Attn}(\bm{\Xi}\bm{W}^Q, \bm{\Delta}\bm{W}^K, \bm{\Delta}\bm{W}^V) \in \mathbb{R}^{N_{an} \times d}, \\
        \bm{Z} = \text{LayerNorm}(\bm{\Xi}\bm{W}^Q + \bm{\Psi}\bm{W}^\Psi), \\
        \bm{Z}^\ast = \text{LayerNorm}(\bm{Z} + \text{FFN}(\bm{Z})),
    \end{gathered}
\end{equation}
where $\bm{W}^Q, \bm{W}^K, \bm{W}^V \in \mathbb{R}^{N_{gene} \times d}$ , and $\bm{W}^\Psi \in \mathbb{R}^{d \times d}$
are trainable weight matrices. The resulting $\bm{Z}^\ast$ contains more expressive AC representations, on which a self-paced clustering \cite{DESC} is conducted to cluster ACs into fine-grained groups. Cell $i$'s soft assignment score to group $j$ is calculated using a Cauchy kernel:
\begin{equation}
    q_{i,j} = \frac{{{\left(1 + \left\| \bm{z}_i^\ast - \bm{\mu}_j \right\|^2 / \nu\right)}}^{-1}}{{\sum_{j^\prime} \left(1 + \left\| \bm{z}_i^\ast - \bm{\mu}_{j^\prime} \right\|^2 / \nu \right)^{-1}}},
\end{equation}
where  $\nu$  denotes the degree of freedom.  $\bm{\mu}_j$ is the centroid of group $j$, initialized by K-means clustering on $\bm{Z}^\ast$. The clustering loss function, a KL-divergence $\mathcal{L}_C$, is calculated on $q$ and an auxiliary target distribution $p$, defined as:
\begin{equation}
        p_{i,j} = \frac{q_{i,j}^2/\sum_{i} q_{i,j}}{\sum_{j}\left(q_{i,j}^2/\sum_{i} q_{i,j}\right)},
\end{equation}

\begin{equation}
    \mathcal{L}_C = \sum_{i}{\sum_{j} p_{i,j}log\left(\frac{p_{i,j}}{q_{i,j}}\right)}.
\end{equation}
$\mathcal{L}_C$ overweights ACs with high-confident cluster assignments in updating the parameters of the cross-attention fusion block and the cluster centroids. Consequently, harder-to-cluster AC embeddings in $\bm{Z}^\ast$ are incrementally transformed into easier ones in iterations, which continues until the changes in anomalies’ hard assignments fall below a threshold or a predetermined number of iterations is reached. The number of AC clusters is either known as a priori or automatically inferred (see Appendix \ref{clustnum}).

\subsection{Theoretical Analysis}
In this section, we provide a theoretical analysis on the existence of $\gamma_c(p_i, p_j)$ in equation \eqref{eq:def_p} and the variation boundary of ${L}_p$ across domains. Following the definitions and notations in section 3.1, we have the following theorem (detailed proof in Appendix \ref{pr3.1}) :
\begin{theorem}\label{th:opt_s}
Let $m$ and $n$ denote the actual numbers of inliers and anomalies in the target dataset, respectively. Let $\bm{\delta}_i \in \bm{\delta}_m^x \cup \bm{\delta}_n^\xi$ , as defined in equation \eqref{eq:def}. Define $s_i \coloneqq f_s(\bm{\delta}_i)\in \{0,1\}$. The equation \eqref{eq:def} can be rewritten as:
\begin{equation}\label{eq:def_s}
    \min_{f_s} \quad \mathcal{L}_p\left(\bm{\delta}_m^x, \bm{\delta}_n^\xi\right) = -\sum_{i}^{m+n}\sum_{j\neq i}^{m+n} k(\bm{\delta}_i, \bm{\delta}_j) \gamma(s_i, s_j),
\end{equation}
where
\begin{equation}\label{eq:gamma}
    \gamma(s_i, s_j) = 
    \begin{cases}
        \frac{1}{m(m-1)}, & s_i = s_j = 0 \\
        \frac{1}{n(n-1)}, & s_i = s_j = 1 \\
        \frac{-1}{mn}, & s_i \neq s_j
    \end{cases}
\end{equation}
If $s_i = 1$, instance $i$ is annotated as anomalous, or normal otherwise. 
\end{theorem}
To ensure a smooth backpropagation of the loss gradients, $f_s$ is replaced with a continuous function $f_p:\mathbb{R}^p \rightarrow (0, 1)$, and  $\gamma(s_i, s_j): \{0, 1\} \times \{0, 1\} \rightarrow \mathbb{R}$ with a continuous mapping function $\gamma_c(p_i, p_j): (0, 1) \times (0, 1) \rightarrow \mathbb{R}$, where $p_i \coloneqq f_p(\bm{\delta}_i)$ is the anomaly score for instance $i$.   \Cref{th:gamma_c}  guarantees the existence of $\gamma_c$ (proof in Appendix \ref{pr3.2}):

\begin{theorem}\label{th:gamma_c}
For any discrete mapping function $\gamma(s_i, s_j): \{0, 1\} \times \{0, 1\} \rightarrow \mathbb{R}$, there exists a continuous function $\gamma_c$ that satisfies:
    \begin{enumerate}[label=(\arabic*)]
        \item $\forall s_i, s_j \in \{0, 1\}$, $\gamma_c(s_i, s_j)=\gamma(s_i, s_j)$.
        \item $\gamma_c$ is analytic everywhere within its domain.
    \end{enumerate}
    For $\gamma$ in \Cref{eq:gamma}, a potential $\gamma_c$ is specified in \Cref{eq:gamma_c}.
\end{theorem}
\Cref{th:trans} addresses the bounded variation of $\mathcal{L}_p$ across domains (proof in Appendix \ref{pr3.3}). 
\begin{theorem}\label{th:trans}
Given \Cref{as:pattern} and a linear kernel-induced MMD, the reconstruction deviations of $m$ inliers ($\bm{\delta}^{x}_m$) and $n$ anomalies ($\bm{\delta}^{\xi}_n$) in any domain satisfy:
\begin{equation}
\begin{split}
& \mathbb{P}\left(\big|\mathcal{L}_p(\bm{\delta}_m^x,\bm{\delta}_n^\xi) - \mathcal{L}_p(P_{\bm{\delta}^{x*}}, P_{\bm{\delta}^{\xi*}})\big|\geq \varepsilon \right) \\ & = \mathbb{P}\left(\big| MMD^2(\bm{\delta}_m^x, \bm{\delta}_n^\xi) - MMD^2(P_{\bm{\delta}^{x*}}, P_{\bm{\delta}^{\xi*}}) \big| \geq \varepsilon \right) \\ & \leq
    \alpha \exp\left(-\frac{\beta C}{1+C}n\varepsilon^2\right),
    \end{split}
\end{equation}
where  $C=\frac{m}{n}\ge 1$, $\alpha, \beta$ are constants, $\varepsilon$ denotes domain-associated variation of $\mathcal{L}_p$. $P_{\bm{\delta}^{x*}}$ and $P_{\bm{\delta}^{\xi*}}$ denote the domain-shift-free distributions of reconstruction deviations of inliers and anomalies, respectively.
\end{theorem}

\begin{table}[htb]
    \renewcommand{\arraystretch}{1.25}  
    \centering
    \resizebox{\linewidth}{!}{
    \begin{tabular}{|c|c|c|c|}
    \hline
    \textbf{Domain Shift Type} & \textbf{Ref Dataset} & \textbf{Target Dataset (Anomaly Type)} & \textbf{Experiment ID} \\
    \hline
    \multirow{20}{*}{\makecell{\large{Experimental} \\ \large{Condition} \\ (\large{\textbf{Type I}})}} & a-TME-$0$ & a-TME-$1$ (Tumor) & 1 \\
    \cline{2-4}
    & \multirow{4}{*}{a-Brain-$0$} & a-Brain-$1$ (Cerebellar granule) & 2 \\
    \cline{3-4}
    & & a-Brain-$2$ (Microglia) & 3 \\
    \cline{3-4}
    & & a-Brain-$1$ (Cerebellar granule) & \multirow{2}{*}{4} \\
    & & a-Brain-$2$ (Microglia) & \\
    \cline{2-4}
    & \multirow{4}{*}{r-PBMC-$0$} & r-PBMC-$1$ (B cell) & 5 \\
    \cline{3-4}
    & & r-PBMC-$2$ (Natural killer) & 6 \\
    \cline{3-4}
    & & r-PBMC-$1$ (B cell) & \multirow{2}{*}{7} \\
    & & r-PBMC-$2$ (Natural killer) & \\
    \cline{2-4}
    & \multirow{4}{*}{r-Cancer-$0$} & r-Cancer-$1$ (Epithelial, Immune tumor) & 8 \\
    \cline{3-4}
    & & r-Cancer-$2$ (Epithelial, Stromal tumor) & 9 \\ 
    \cline{3-4}
    & & r-Cancer-$1$ (Epithelial, Immune tumor) & \multirow{2}{*}{10} \\
    & & r-Cancer-$2$ (Epithelial, Stromal tumor) & \\
    \cline{2-4}
    & \multirow{8}{*}{r-CSCC-$0$} & r-CSCC-$1$ (\centering \makecell{Basal, Cycling, \\ Keratinocyte tumor}) & 11 \\
    \cline{3-4}
    & & r-CSCC-$2$ (\centering \makecell{Basal, Cyclng, \\ Keratinocyte tumor}) & 12 \\    
    \cline{3-4}
    & & r-CSCC-$1$ (\centering \makecell{Basal, Cycling, \\ Keratinocyte tumor}) & \multirow{4}{*}{\centering 13} \\
    & & r-CSCC-$2$ (\centering \makecell{Basal, Cycling, \\ Keratinocyte tumor}) & \\
    & & r-CSCC-$3$ (\centering \makecell{Basal, Cycling, \\ Keratinocyte tumor}) & \\
    \hline
    \multirow{4}{*}{\makecell{\large{Sequencing} \\ \large{Technology} \\ (\large{\textbf{Type II})}}} & \multirow{4}{*}{r-Pan-$0$} & r-Pan-$1$ (Delta, Acinar cell) & 14 \\
    \cline{3-4}
    & & r-Pan-$2$ (Delta, Acinar cell) & 15 \\
    \cline{3-4}
    & & r-Pan-$1$ (Delta, Acinar cell) & \multirow{2}{*}{16} \\
    & & r-Pan-$2$ (Delta, Acinar cell) & \\
    \hline
    \makecell{\large{Data Type} \\ (\large{\textbf{Type III}})} & r-PBMC-$0$ & a-TME-$1$ (Tumor) & 17 \\
    \hline
    \multirow{2}{*}{\makecell{\large{Protocol Type} \\ (\large{\textbf{Type IV}})}} & \multirow{2}{*}{KDDRev-$0$} & KDDRev-$1$ (DOS,PROBING) & \multirow{2}{*}{18} \\
    & & KDDRev-$2$ (DOS,R2L,U2R,PROBING) & \\
    \hline
    \end{tabular}
}
    \caption{Experimental scenarios. For experiments 1-17, the dataset name is formatted as (data type)-(data name)-(dataset ID). For data type, a and r denote scATAC-seq and scRNA-seq, respectively. Dataset ID $0$ is reserved for reference datasets.}
    \label{tab:tab1}
\end{table}
\vspace{-0.5cm}

\begin{table*}[htb]
\centering
\fontsize{8pt}{8pt}\selectfont
\renewcommand{\arraystretch}{1.5}  
\resizebox{\linewidth}{!}{
\begin{tabular}{c|*{17}{c}}
\toprule

\multicolumn{1}{c|}{\multirow{2}{*}{\large{\textbf{Method}}}} & \multicolumn{17}{c}{\large{\textbf{Experiment ID}}} \\
\cline{2-18} 
& \textbf{1} & \textbf{2} & \textbf{3} & \textbf{4} & \textbf{5} & \textbf{6} & \textbf{7} & \textbf{8} & \textbf{9} & \textbf{10} & \textbf{11} & \textbf{12} & \textbf{13} & \textbf{14} & \textbf{15} & \textbf{16} & \textbf{17} \\
\hline
\textit{SLAD} & .52(.02) & \underline{.70(.00)} & .63(.01) & .58(.01) & \underline{.77(.01)} & \textbf{.88(.01)} & \underline{.82(.01)} & .93(.00)& .92(.01) & .93(.01) & .59(.01) & .68(.02) & .60(.02) & \underline{.78(.02)} & .81(.04) & .62(.01) & .49(.03) \\
\textit{ICL} & \underline{.78(.06)} & .70(.01) & .54(.01) & .50(.01) & .65(.03) & .72(.04) & .68(.03) & .93(.01) & .93(.01) & .93(.01) &.37(.01) & .33(.01) & .14(.02) & .77(.02) & \underline{.84(.01)} & .62(.01) & .50(.05) \\
\textit{NeuTraL} & .49(.02) & .67(.01) & .70(.02) & .56(.01) & .71(.01) & .74(.02) & .71(.01)& \underline{.96(.01)} & \underline{.95(.03)} & \underline{.95(.01)} & .25(.05) & .20(.04) & .19(.04) & .53(.04) & .47(.06) & .51(.05) & .59(.07) \\
\textit{RCA} & .37(.04) & .69(.02) & \underline{.77(.03)} & \underline{.69(.01)} & .72(.01) & .79(.01) & .77(.02) & .84(.03) & .75(.05) & .81(.04) & .33(.02) & .33(.03) & .21(.01) & .66(.04) & .72(.02) & \underline{.81(.02)} & .37(.04) \\
\textit{CAMLU} & .66(.12) & .69(.02) & \textbf{.79(.09)} & .55(.02) & .70(.05) & .82(.04) & .71(.05) & .90(.00) & .85(.13) & .88(.01) & \underline{.80(.02)} & \textbf{.79(.00)} & \textbf{.78(.01)} & .55(.06) & .62(.03) & .53(.05) & .58(.00) \\
\textit{ALAD} & .36(.18) & .40(.09) & .46(.18) & .45(.08) & .42(.05) & .33(.08) & .33(.02) & .22(.07) & .16(.08) & .20(.08) & .53(.09) & .48(.13) & .51(.13) & .33(.04) & .41(.05) & .38(.04) & .60(.07) \\
\textit{DeepSVDD} & .57(.02) & .55(.01) & .61(.02) & .51(.01) & .62(.02) & .42(.02) & .48(.02) & .44(.14) & .64(.04) & .56(.08) & .47(.03) & .53(.02) & .50(.02) & .62(.01) & .43(.06) & .53(.03) & \underline{.62(.03})  \\
\textit{scmap} & .50(.00) & .59(.00) & .56(.00) & .57(.00) & .70(.00) & .83(.00) & .76(.00) & .53(.00) & .47(.00) & .50(.00) & .57(.00) & .65(.00) & .67(.00) & .77(.00) & .54(.00) & .78(.00) & .50(.00) \\

\hline
\textit{ACSleuth}  & \textbf{.83(.02)} & \textbf{.74(.02)} & .72(.03) & \textbf{.71(.01)} & \textbf{.78(.03)} & \underline{.85(.00)} & \textbf{.83(.02)} & \textbf{.97(.01)} & \textbf{.95(.02)} & \textbf{.96(.01)} & \textbf{.81(.02)} & \underline{.70(.02)} & \underline{.69(.03)} & \textbf{.85(.02)} & \textbf{.93(.03)} & \textbf{.88(.03)} & \textbf{.74(.03)} \\
\bottomrule
\end{tabular}
}
    \caption{The AC detection performance across 17 experiments is evaluated using averaged AUC scores with standard deviations indicated in parentheses. The best and second best scores for each experiment are \textbf{bolded} and \underline{underlined}, respectively.}
    \label{tab:tab2}
\end{table*}

\begin{table*}[htb]
    \centering
    \fontsize{8pt}{8pt}\selectfont
    \renewcommand{\arraystretch}{1.3}
    \resizebox{\textwidth}{!}{
\begin{tabular}{c|*{11}{c}}
\toprule

\multicolumn{1}{c|}{\multirow{2}{*}{\textbf{Method}}} & \multicolumn{11}{c}{\textbf{Experiment ID}} \\
\cline{2-12} 
& \textbf{4} & \textbf{7} & \textbf{8} & \textbf{9} & \textbf{10} & \textbf{11} & \textbf{12} & \textbf{13} & \textbf{14} & \textbf{15} & \textbf{16} \\
\hline
\textit{scGAD} & .18(.01) & .13(.08) & .51(.02) & .55(.00) 
& \underline{.42(.01)} & \underline{.27(.01)} & .18(.00) & .16(.01) & .11(.01) 
&.16(.01) & .13(.01) \\
\textit{MARS} & .19(.01) & .28(.02) & .34(.00) & .40(.00) & .28(.02) & .23(.01) & \underline{.24(.01)} & \underline{.23(.01)} & \underline{.21(.01)} & .27(.02) & \underline{.20(.01)} \\
\textit{SLAD-EDESC} & .19(.06) & .27(.13) & .61(.28) & \underline{.64(.05)} & .36(.14) & .17(.02) & .05(.04) & .05(.04) & .10(.05) & .15(.10) & .17(.10) \\
\textit{SLAD-DFCN} & \underline{.26(.01)} & \underline{.41(.02)} & \underline{.71(.02)} & .60(.01) & .37(.04) & .08(.03) & .04(.02) & .05(.03) & .05(.04) & .29(.04) & .02(.01) \\
\textit{SLAD-scTAG} & .07(.05) & .10(.05) & .53(.05) & .50(.15) & .23(.07) & .20(.02) & .03(.02) & .02(.01) & .03(.01) & .11(.01) & .16(.05) \\
\textit{SLAD-Leiden} & .11(.01) & .25(.02) & .10(.04) & .07(.03) & .27(.10) & .20(.00) & .13(.01) & .02(.00) & .20(.02) & \underline{.31(.03)} & .13(.01)
\\
\textit{SLAD-Seurat V5} & .11(.01) & .40(.02) & .31(.02) & .27(.01) & .31(.01) & .22(.01) & .04(.01) & .03(.00) & .02(.01) & .00(.00) & .00(.00) \\

\hline
\textit{ACSleuth}  & \textbf{.36(.03)} & \textbf{.47(.06)} & \textbf{.79(.02)} & \textbf{.69(.03)} & \textbf{.74(.03)} & \textbf{.61(.02)} & \textbf{.52(.07)} & \textbf{.58(.04)} & \textbf{.31(.07)} & \textbf{.65(.04)} & \textbf{.43(.06)} \\
\bottomrule
    \end{tabular}
}
    \caption{The overall FACD performance across 11 experiments, evaluated using average F1$\times$NMI scores with standard deviations indicated in parentheses. The best and second best scores for each experiment are \textbf{bolded} and \underline{underlined}, respectively.}
    \label{tab:tab3}
\end{table*}

\section{Experiments}
\subsection{Experimental Settings}
\paragraph{Datasets.} scRNA-seq datasets used in this study include three 10x Peripheral Blood Mononuclear Cell (PBMC) datasets, three 10x lung cancer (Cancer) datasets, four 10x Cutaneous Squamous Cell Carcinoma (CSCC) datasets, and three Pancreatic (Pan) datasets sequenced with different technologies (InDrop, Smart-Seq2, and CEL-Seq2). scATAC-seq datasets include two 10x Human Tumor Microenvironment (TME) datasets and three 10x Mouse Brain (Brain) datasets. We use the KDDRev dataset for fine-grained cyber-intrusion detection. See Appendix \ref{dataset} for more details. 

\paragraph{Experimental Scenarios.} Experiments are carefully designed for scenarios regarding DS types, target dataset quantity, and dataset-specific AC types (\cref{tab:tab1}). For SC data (i.e., experiments 1-17), DS is categorized into three types based on their origins: variations in experimental conditions (\textbf{Type I}), sequencing technologies (\textbf{Type II}), and data types (\textbf{Type III}). For cyber-intrusion detection data (i.e., experiment 18), DS stems from variations in connection protocol types (\textbf{Type IV}). Experiments 1-3, 5-6, 8-9, 11-12, 14-15, and 17 focus on a single target dataset. In cases involving more than one target dataset, experiments 13 and 16 feature target datasets with identical AC types, experiments 10 and 18 include both shared and dataset-specific anomaly types, while experiments 4 and 7 are set up with exclusively dataset-specific AC types.

\paragraph{Baselines.} The baselines for AC detection cover the SOTA methods across three major categories: contrastive methods (ICL \cite{ICL} and NeuTraL \cite{NeuTraL}), generative methods (CAMLU \cite{CAMLU}, RCA \cite{RCA}, and ALAD \cite{ALAD}), and one-class classification method (DeepSVDD \cite{DeepSVDD18}). We also include SLAD \cite{SLAD}, a cutting-edge scale-learning method, and scmap \cite{scmap}, an uncertainty-based method tailored for scRNA-seq. For overall FACD, the baselines encompass five composite methods that combine SLAD for AC detection with EDESC \cite{EDESC}, DFCN \cite{DFCN}, scTAG \cite{scTAG}, Leiden \cite{leiden}, and Seurat V5 \cite{Seurat} for fine-grained AC annotation, where the first three are deep clustering methods, and the last two are for SC clustering. We also include MARS \cite{MARS} and scGAD \cite{scGAD},  the only two methods available that integrate AC detection and fine-grained annotation. For fine-grained cyber-intrusion detection, MARS, scGAD, and three composite methods, SLAD-EDESC, SLAD-DFCN, and SLAD-Leiden, serve as baselines.

\paragraph{Evaluation Protocols.} 
Both the AUC and F1 scores are used to evaluate AD accuracy. For a fair comparison across methods, the F1 score is calculated with a deliberately chosen anomaly score threshold such that the number of samples exceeding this threshold (i.e., labeled as anomalies) matches the actual number of true anomalies \cite{ICL}. The normalized mutual information (NMI) is used to evaluate fine-grained anomaly annotation in both SC and cyber-intrusion data. As the overall performance of anomaly analysis hinges on both anomaly detection and fine-grained annotation, it is evaluated using the product of the F1 and NMI scores. The reported metrics represent averaged results, along with standard deviations obtained over ten independent runs.

\paragraph{Implementation Details.} GAN models in \textit{module I} and \textit{module II} share the same architecture—the encoder: Input-512-256-256-256-256-256; the decoder is symmetric to the encoder; the discriminator: Input-512-64-64-64. The MMD-based AC scorer in \textit{module I} is a three-layer MLP with an Input-512-256-1 configuration. The memory block in \textit{module I} has a $512\times 256$ dimension. The style block in \textit{module II} has a $N_{tgt}\times 256$ dimension, where $N_{tgt}$ is the number of target datasets. The weight parameters $\alpha$, $\beta$, and $\lambda$ in GAN's loss functions are set to 50, 1, and 10, respectively. The cross-attention block in Phase III includes two attention heads of dimension 256. The mini-batch size of two GAN models is 256, while the MMD-based AC scorer and the clustering module utilize the full batch. The training is conducted using an Adam optimizer with a default learning rate of 3e-4.

\subsection{Results}
\paragraph{Anomalous Cell Detection.} \Cref{tab:tab2} showcases ACSleuth's superiority over the competing methods in accurate AC detection across various scenarios (\Cref{tab:tab1}). ACSleuth ranks first 13 times, and among the top two performers 16 times. Significantly, ACSleuth consistently leads in experiments (4, 7, and 10) involving multiple target datasets and dataset-specific AC types. Moreover, the addition of a new target dataset (a-Brain-$1$) with its unique AC type in experiment 4, compared to experiment 3, highlights ACSleuth as the only method that maintains its performance level amidst the newly introduced DS and dataset-specific AC type. Furthermore, ACSleuth significantly outperforms other methods in experiments 14-17 featuring DS (\textbf{Type II} and \textbf{III}), with an average leap of 11.92\% in AUC scores. This underscores the essential role of domain adaptation in ensuring accurate AC detection in such complex scenarios. Collectively, these findings affirm ACSleuth's efficacy in AD, especially in scenarios involving significant DS and dataset-specific AC types.

\paragraph{Overall Performance of Fine-grained Anomalous Cell Detection.} Here, the overall performance of ACSleuth in FACD is compared to the baseline methods, using the product of F1 (for AC detection) and NMI (for fine-grained AC annotating) scores as the evaluation metric. \Cref{tab:tab3} showcases that ACSleuth consistently outperforms the top-performing baseline methods across 11 experiments, achieving an average leap of 74.13\% in F1$\times$NMI scores. These results reveal several specific strengths of our method: Experiments 4 and 7-13 demonstrate ACSleuth's excelling in handling DS \textbf{Type I}, while experiments 14-16 showcase its robustness against DS \textbf{Type II}. Moreover, experiments 4, 7, and 10 highlight ACSleuth's ability to manage complications arising from sample-specific AC types. In contrast, all baseline methods exhibit suboptimal performance, primarily due to their lack of a domain adaptation step. Although scGAD integrates MNN for domain adaptation, its performance is compromised in scenarios involving sample-specific AC types (see \textbf{Case ii} in \Cref{fig:illustration}), as indicated in experiments 4 and 7. Finally, ACSleuth, scGAD, and MARS, which unify the AC detection and fine-grained annotation, emerge as at least two among the top three performers in seven out of the eleven experiments, suggesting the importance of maintaining methodological coherence of the workflow.

\begin{table}[htb]
    \centering
    \fontsize{10pt}{10pt}\selectfont
    \renewcommand{\arraystretch}{1.5}  
    \setlength{\tabcolsep}{1pt}  
    \resizebox{\linewidth}{!}{
    \begin{tabular}{|c|c|c|c|c|c|c|c|}
    \hline
    \multirow{2}{*}{\centering \makecell{\textbf{Experiment} \\ \textbf{ID}}} & \multirow{2}{*}{\centering \makecell{\textbf{Metrics}}} & \multirow{2}{*}{\textit{ACSleuth}} & \multirow{2}{*}{\textit{scGAD}} & \multirow{2}{*}{\textit{MARS}} & \multicolumn{3}{c|}{\Large{\textit{SLAD}}} \\ \cline{6-8}
    & & & & & \textit{EDESC} & \textit{DFCN} & \textit{Leiden} \\
    \hline
    \multirow{3}{*}{18} & F1&
    .72(.02) & .41(.08) & .70(.01) & \multicolumn{3}{c|}{.44(.02)} \\
    \cline{2-8}
    & NMI & .73(.02) & .49(.05) & .70(.01) & .47(.06) & .32(.04) & .50(.00) \\
    \cline{2-8}
    & F1$\times$NMI & .53(.03) & .20(.04) & .49(.01) & .20(.02) & .15(.01) & .22(.01) \\
    \hline
    \end{tabular}
}
    \caption{The detection and differentiation of cyber-intrusions are evaluated using F1 and NMI scores, respectively.  Overall performance is measured in F1$\times$NMI scores. Each cell displays the average score of 10 runs, with standard deviation noted in parenthesis.}
    \label{tab:tab4}
\end{table}
\vspace{-0.5cm}

\paragraph{Fine-grained Cyber-intrusion Detection.} Here, we assess ACSleuth's applicability to fine-grained cyber-intrusion detection in the presence of DS rooted in variations in connection protocol types. The reference dataset includes UDP records only, while the two target datasets contain ICMP and TCP records, respectively. \Cref{tab:tab4} showcases ACSleuth outperforms the competing methods in detecting and distinguishing various types of cyber-intrusions in terms of F1$\times$NMI scores.
Intriguingly, since R2L and U2R anomaly types are unique to KDDRev-2, using domain adaptation methods like MNN, which are sensitive to complications caused by dataset-specific anomaly types, can deteriorate performance in both AD and fine-grained annotation (see \textbf{Case ii} in \Cref{fig:illustration}). This is supported by scGAD's low-performance ranking in both tasks. Conversely, ACSleuth's superior performance suggests its domain adaptation strategy is robust against such complications. Furthermore, these results also indicate ACSleuth's greater versatility, compared to MARS and scGAD, in handling general tabular data.

\subsection{Ablation Study}
A series of ablation studies are conducted to evaluate the roles of ACSleuth's key components in AC detection and fine-grained annotating (\Cref{tab:ablation}). The MMD-based anomaly scorer's efficacy is assessed by substituting it with two other commonly used anomaly scoring functions in reconstruction deviation-guided generative models \cite{AnoGAN,ALAD}: $F_G(\bm{x}_i) \coloneqq \Vert \bm{x}_i - \widehat{\bm{x}}_i \Vert_2$ and $F_D(\bm{x}_i) \coloneqq \Vert D^I(\bm{x}_i) - D^I(\widehat{\bm{x}}_i) \Vert_2$, where $\widehat{\bm{x}}_i$ is the reconstructed $\bm{x}_i$ by \textit{module I}. Our scorer outperforms $F_G$ and $F_D$ in detecting ACs, as evidenced by its average AUC improvement of 12.07\% over $F_G$ and 22.16\% over $F_D$. Removing the memory block in \textit{module I} (w/o MB) leads to a notable decrease in anomaly detection accuracy, averaging a 19.89\% reduction in AUC scores. The essential role of the domain adaptation step in Phase II (w/o DA) in ensuring accurate fine-grained AC annotation is reflected by the observation that its omission leads to an average decline of 36.49\% in F1$\times$NMI scores. Lastly, using solely domain-adapted cell embeddings  (w/o FB), rather than their fusions with reconstruction deviations, for the clustering in Phase III reduces F1$\times$NMI scores by 14.84\% on average.

\begin{table}[htb]
\centering
\renewcommand{\arraystretch}{1.2}

\resizebox{\linewidth}{!}{
\begin{tabular}{|c|c|c|c|c|c|}
\hline
\multirow{2}{*}{\textbf{Task}} & \multirow{2}{*}{\textbf{Ablation}} & \multicolumn{4}{c|}{\large{\textbf{Experiment ID}}} \\
\cline{3-6}
& & \textbf{4} & \textbf{10} & \textbf{13} & \textbf{16} \\ 
\hline
\multirow{4}{*}{\centering ACD} & w/ $F_G$ & .59(.02) & .88(.03) & .66(.01) & .77(.02) \\
& w/ $F_D$ & .64(.05) & .58(.08) & .65(.05) & .83(.08) \\
& w/o MB & .67(.01) & .89(.03) & .61(.02) & .76(.02) \\ \cline{2-6}
& Full & .71(.01) & .96(.04) & .69(.03) & .88(.03) \\ \hline \hline
\multirow{3}{*}{\centering FACD} & w/o DA & .22(.06) & .50(.11) & .39(.09) & .25(.08) \\
& w/o FB & .34(.02) & .65(.07) & .46(.05) & .34(.03) \\ \cline{2-6}
& Full & .36(.03) & .74(.03) & .58(.04) & .43(.06) \\ \hline
\end{tabular}
}
\caption{The upper table shows ACD results in average AUC scores. The lower table shows overall FACD results in average F1$\times$NMI scores. ``Full'' is the complete ACSleuth model. ``w/ $F_G$'' and ``w/ $F_D$'' substitute the MMD-based scorer with the $F_G$ and $F_D$ scoring functions, respectively. ``w/o MB'' removes the memory block from \textit{module I}. ``w/o DA'' omits the domain adaptation phase. ``w/o FB''  uses domain-adapted cell embeddings only for clustering. }
\label{tab:ablation}
\end{table}
\vspace{-0.5cm}

\section{Conclusion}
In this paper, We propose ACSleuth, a novel method for FACD in multi-sample and multi-domain contexts. ACSleuth provides a generative framework that integrates AC detection, domain adaptation, and fine-grained annotation into a methodological cohesive workflow. Our theoretical analysis corroborates ACSleuth's resilience against DS. Extensive experiments, designed for scenarios involving various real datasets and DS types, have shown ACSleuth's superiority over the SOTA methods in identifying and differentiating ACs. Our method is also versatile for tabular data types beyond SC.

\section*{Acknowledgments}
The project is funded by the Excellent Young Scientist Fund of Wuhan City (Grant No. 21129040740) to X.S.




\appendix
\onecolumn
\title{\vspace{-2cm}\textbf{Domain Adaptive and Fine-grained Anomaly Detection for Single-cell Sequencing Data and Beyond} \\
        ~\\
       \textbf{Supplementary Material}}

\author{}
\date{}



\section{Theorem Proofs}
\subsection{Proof of Theorem 3.1}
\begin{proof}\label{pr3.1}
Gretton et al \cite{gretton2012kernel} provide an unbiased empirical MMD for samples:
\begin{equation}\label{eq:MMD_XY}
    \begin{aligned}
        & MMD^2\left(\bm{\delta}_m^x, \bm{\delta}_n^\xi\right)\\
        & = \frac{1}{m(m-1)}\sum_{i}^{m}\sum_{j\neq i}^{m}k(\bm{\delta}_i^x, \bm{\delta}_j^x)
          + \frac{1}{n(n-1)}\sum_{i}^{n}\sum_{j\neq i}^{n}k(\bm{\delta}_i^\xi, \bm{\delta}_j^\xi) - \frac{2}{mn}\sum_{i}^{m}\sum_{j}^{n}k(\bm{\delta}_i^x, \bm{\delta}_j^\xi) \\
        & = \sum_{i}^{m+n}\sum_{j\neq i}^{m+n}k(\bm{\delta}_i, \bm{\delta}_j)\gamma(s_i, s_j),
    \end{aligned}
\end{equation}
where the adjustment coefficient $\gamma$ is defined as:
\begin{equation}
    \gamma(s_i, s_j) =
    \begin{cases}
        \frac{1}{m(m-1)}, & s_i = s_j = 0. \\
        \frac{1}{n(n-1)}, & s_i = s_j = 1. \\
        \frac{-1}{mn}, & s_i \neq s_j.
    \end{cases}
\end{equation}
If $s_i = 1$, instance $i$ is annotated as anomalous, or normal otherwise. 
\end{proof}

\setcounter{section}{1}
\subsection{Proof of Theorem 3.2}\label{proof:gamma_c}
\begin{proof}\label{pr3.2}
We define the two-dimensional sequence $\{\gamma_{m,n}\}_{m,n\in\mathbb{Z}}$ as:
\begin{equation}
    \gamma_{0,0} = \frac{1}{m(m-1)}, \quad \gamma_{0,1} = \gamma_{1,0} = \frac{-1}{mn}, \quad \gamma_{1,1} = \frac{1}{n(n-1)}, \quad \forall i, j \geq 2, \gamma_{i,j} = 0
\end{equation}
The ordinary generating function $H(x, y)$ for $\gamma_{m,n}$ is:
\begin{equation}
    H(x, y) = \sum_{i=0}^{\infty}\sum_{j=0}^{\infty}\gamma_{i,j}x^iy^j,
\end{equation}
where $(x, y) \in \mathbb{D} \coloneqq [0, 1]\times [0, 1]$.
In this case, $H(x, y)$ is a formal power series, and $\gamma_{ij}$ corresponds to the coefficients of $x^iy^j$. Note that $H(-x, -y)$ satisfies:
\begin{equation}
    H(-x, -y) = \sum_{i=0}^{\infty}\sum_{j=0}^{\infty}\gamma_{ij}\Gamma(i+1)\Gamma(j+1)\frac{(-1)^i(-1)^j}{i!j!}x^iy^j.
\end{equation}
By \Cref{le:MultiRam}, which is the extension of Ramanujan's master theorem in the $k$-dimensional case \cite{bradshaw2023operational,amdeberhan2012ramanujan}, the $k$-dimensional Mellin transform follows:
\begin{equation}
    \begin{aligned}
        \mathcal{M}[H(-x, -y)](s, t) & \coloneqq \int_{\mathbb{D}} x^{s-1}y^{t-1} H(-x, -y) dxdy \\
        & = \Gamma(s)\Gamma(t)\Gamma(1-s)\Gamma(1-t) \gamma_c(-s, -t)\\
        & = \frac{\pi^2}{\sin\pi s \sin\pi t} \gamma_c(-s, -t),
    \end{aligned}
\end{equation}
where $\gamma_c(s, t)$ is exactly the extension of sequence $\gamma_{ij}$ in the continuous scenario. By solving the definite integral, we can obtain:
\begin{equation}
    \begin{aligned}
        \gamma_c(-s, -t) & = \frac{\sin\pi s \sin\pi t}{\pi^2} \int_{\mathbb{D}} x^{s-1}y^{t-1} H(-x, -y) dxdy \\
        & = \frac{\sin\pi s \sin\pi t}{\pi^2} \int_{\mathbb{D}} \sum_{i=0}^{\infty}\sum_{j=0}^{\infty}\gamma_{ij}x^{i+s-1}y^{j+t-1} dxdy \\
        & = \frac{\sin\pi s \sin\pi t}{\pi^2} \int_{\mathbb{D}} (\gamma_{0,0}x^{s-1}y^{t-1} + \gamma_{1,0}x^{s}y^{t-1} + \gamma_{0,1}x^{s-1}y^{t} + \gamma_{1,1}x^sy^t) dxdy \\
        & = \frac{\sin\pi s \sin\pi t}{\pi^2} \left(\gamma_{0,0}\frac{x^s}{s}\bigg|_0^1\frac{y^t}{t}\bigg|_0^1 + \gamma_{1,0}\frac{x^{s+1}}{s+1}\bigg|_0^1\frac{y^t}{t}\bigg|_0^1 + \gamma_{0,1}\frac{x^s}{s}\bigg|_0^1\frac{y^{t+1}}{t+1}\bigg|_0^1 + \gamma_{1,1}\frac{x^{s+1}}{s+1}\bigg|_0^1\frac{y^{t+1}}{t+1}\bigg|_0^1 \right) \\
        & = \frac{\sin\pi s \sin\pi t}{\pi^2} \left(\frac{[m(m-1)]^{-1}}{st}+\frac{(mn)^{-1}}{(s+1)t}+\frac{(mn)^{-1}}{s(t+1)}+\frac{[n(n-1)]^{-1}}{(s+1)(t+1)}. \right)
    \end{aligned}
\end{equation}
By replacing $-s$ and $-t$, \Cref{Theorem 3.2} is proven.
\end{proof}

\setcounter{section}{1}
\subsection{Proof of Theorem 3.3}

\begin{proof}\label{pr3.3}
Let $S^r$ and $S^t$ denote the sample sets in reference domain $r$  and target domain $t$, respectively. $\bm{\zeta} \in S^r$ denote a reference sample. Given a proper reconstruction function $G(\cdot)$ trained on $S^r$ exclusively, under Assumption 3.1 in the main text, $\bm{\zeta}$ and its reconstruction $\widehat{\bm{\zeta}}$ satisfy: 
\begin{equation}
    \begin{cases}
        & \bm{\zeta} = \bm{\zeta}^* + \bm{b}^r + \bm{\epsilon}, \\
        &  \widehat{\bm{\zeta}} \coloneqq G(\bm{\zeta}) = \widehat{\bm{\zeta}}^* + \bm{b}^r,
    \end{cases}
\end{equation}
where $\widehat{\bm{\zeta}}^* \sim P_{\widehat{\bm{\zeta}}^*}, \bm{b}^r \sim P_{\bm{b}^r}$. The random error term $\bm{\epsilon}$ is ignored in  $\widehat{\bm{\zeta}}$, as it cannot be learned by $G(\cdot)$. Given an inlier $\bm{x}_i$ and anomaly $\bm{\xi}_j $ in $S^t$, $\exists \bm{\zeta}_i, \bm{\zeta}_j \in S^r$ whose reconstructions by $G(\cdot)$ are equivalent to those of $\bm{x}_i$ and $\bm{\xi}_j $ :
\begin{equation}
    \begin{cases}
        & \bm{x}_i = \bm{x}_i^* + \bm{b}^t + \bm{\epsilon}_i, \\
        & \widehat{\bm{x}}_i = G(\bm{x}_i) = G(\bm{\zeta}_i) = \widehat{\bm{\zeta}}_i^* + \bm{b}^r, \\
        & \bm{\xi}_j = \bm{\xi}_j^* + \bm{b}^t + \bm{\epsilon}_j, \\
        & \widehat{\bm{\xi}}_j = G(\bm{\xi}_j) = G(\bm{\zeta}_j) = \widehat{\bm{\zeta}}_j^* + \bm{b}^r ,
    \end{cases}
\end{equation}
The reconstruction deviations of $\bm{x}_i$ and $\bm{\xi}_j $ satisfy:
\begin{equation}
    \begin{cases}
        & \bm{\delta}_i^x = \bm{x}_i - \widehat{\bm{x}}_i = \bm{x}_i^* - \widehat{\bm{\zeta}}_i^* + \bm{b}^t - \bm{b}^r + \bm{\epsilon}_i, \\
        & \bm{\delta}_j^\xi = \bm{\xi}_j - \widehat{\bm{\xi}}_j = \bm{\xi}_j^* - \widehat{\bm{\zeta}}_j^* + \bm{b}^t - \bm{b}^r + \bm{\epsilon}_j.
    \end{cases}
\end{equation}
For a more concise representation, we define:
\begin{equation}
    \begin{cases}
        & \bm{\delta}_i^{x*} = \bm{x}_i^* - \widehat{\bm{\zeta}}_i^*, \quad \bm{\delta}_i^b = \bm{b}^t - \bm{b}^r + \bm{\epsilon}_i. \\
        & \bm{\delta}_j^{\xi*} = \bm{\xi}_j^* - \widehat{\bm{\zeta}}_j^*, \quad \bm{\delta}_j^b = \bm{b}^t - \bm{b}^r + \bm{\epsilon}_j.
    \end{cases}
\end{equation}
Given a linear kernel $k$, it can be expanded as follows:
\begin{equation}
    \begin{aligned}
        & k(\bm{\delta}_i^x, \bm{\delta}_j^x) = k(\bm{\delta}_i^{x*}, \bm{\delta}_j^{x*}) + k(\bm{\delta}_i^b, \bm{\delta}_j^b) + k(\bm{\delta}_i^{x*}, \bm{\delta}_j^b) + k(\bm{\delta}_i^b, \bm{\delta}_j^{x*}), \\
        & k(\bm{\delta}_i^\xi, \bm{\delta}_j^\xi) = k(\bm{\delta}_i^{\xi*}, \bm{\delta}_j^{\xi*}) + k(\bm{\delta}_i^b, \bm{\delta}_j^b) + k(\bm{\delta}_i^{\xi*}, \bm{\delta}_j^b) + k(\bm{\delta}_i^b, \bm{\delta}_j^{\xi*}), \\
        & k(\bm{\delta}_i^x, \bm{\delta}_j^\xi) = k(\bm{\delta}_i^{x*}, \bm{\delta}_j^{\xi*}) + k(\bm{\delta}_i^b, \bm{\delta}_j^b) + k(\bm{\delta}_i^{x*}, \bm{\delta}_j^b) + k(\bm{\delta}_i^b, \bm{\delta}_j^{\xi*}).
    \end{aligned}
\end{equation}
Following \eqref{eq:MMD_XY}, we have:
\begin{equation}\label{eq:MMDMMD}
    \begin{aligned}
        & MMD^2\left(\bm{\delta}_m^x, \bm{\delta}_n^\xi\right) \\
        & = \frac{1}{m(m-1)}\sum_{i}^{m}\sum_{j\neq i}^{m}k(\bm{\delta}_i^x, \bm{\delta}_j^x) + \frac{1}{n(n-1)}\sum_{i}^{n}\sum_{j\neq i}^{n}k(\bm{\delta}_i^\xi, \bm{\delta}_j^\xi) - \frac{2}{mn}\sum_{i}^{m}\sum_{j}^{n}k(\bm{\delta}_i^x, \bm{\delta}_j^\xi) \\
        & = MMD^2\left(\bm{\delta}_m^{x*}, \bm{\delta}_n^{\xi*}\right) + MMD^2\left(\bm{\delta}_m^b, \bm{\delta}_n^b\right) + 2R_{mn}^x + 2R_{mn}^\xi.
    \end{aligned}
\end{equation}
where $\bm{\delta}^{\cdot}_m=\{\bm{\delta}^{\cdot}_i | \bm{\delta}^{\cdot}_i\in \mathbb{R}^d,\cdot \in \{x,x*,b\},i=1,2,\cdots,m\}$ and $\bm{\delta}^{\cdot}_n=\{\bm{\delta}^{\cdot}_j | \bm{\delta}^{\cdot}_j\in \mathbb{R}^d,\cdot \in \{\xi,\xi*,b\},j=1,2,\cdots,n\}$. The remainder terms $R_{mn}^x$ and $R_{mn}^\xi$ can be expressed as:
\begin{equation}
    \begin{aligned}
        R_{mn}^x & \coloneqq \frac{1}{m(m-1)}\sum_{i}^{m}\sum_{j\neq i}^{m}{\bm{\delta}_i^b} \cdot \bm{\delta}_j^{x*} - \frac{1}{n^2}\sum_{i}^{n}\sum_{j}^{n}{\bm{\delta}_i^b} \cdot \bm{\delta}_j^{x*}, \\
        R_{mn}^\xi & \coloneqq \frac{1}{n(n-1)}\sum_{i}^{m}\sum_{j\neq i}^{m}{\bm{\delta}_i^b} \cdot \bm{\delta}_j^{\xi*} - \frac{1}{m^2}\sum_{i}^{m}\sum_{j}^{m}{\bm{\delta}_i^b} \cdot \bm{\delta}_j^{\xi*}.
    \end{aligned}
\end{equation}

Next, we discuss the asymptotic convergence property of each term in equation \eqref{eq:MMDMMD}.
For the first and second terms, following \Cref{le:MMDcon} \cite{gretton2012kernel}, we have:
\begin{equation}\label{eq:1st2nd}
    \begin{gathered}
        \mathbb{P}\left(\big|MMD^2(\bm{\delta}_m^{x*}, \bm{\delta}_n^{\xi*}) - MMD^2(P_{\bm{\delta}^{x*}}, P_{\bm{\delta}^{\xi*}}) \big| \geq \varepsilon \right) \leq 2\exp\left(\frac{-Cn\varepsilon^2}{8(1+C){K_+^x}^2}\right), \\
        \mathbb{P}\left(\big|MMD^2(\bm{\delta}_m^{b}, \bm{\delta}_n^{b}) - 0 \big| \geq \varepsilon \right) \leq 2\exp\left(\frac{-Cn\varepsilon^2}{8(1+C){K_+^b}^2}\right),
    \end{gathered}
\end{equation}
where $K_+^x \coloneqq \sup_{ij} k(\bm{\delta}_i^{x*}, \bm{\delta}_j^{x*})$, $K_+^b \coloneqq \sup_{ij} k(\bm{\delta}_i^{b}, \bm{\delta}_j^{b})$.

We then discuss the asymptotic convergence property of $R_{mn}^x$ in \eqref{eq:MMDMMD}. Given $\bm{\delta}^b$ and $\bm{\delta}^{x*}$ are independent, we define a random variable $y \coloneqq {\bm{\delta}^b} \cdot \bm{\delta}^{x*} \in \mathbb{R}$. According to \Cref{le:1/nepsilon} and Hoeffding's Inequality \cite{hoeffding1994probability}, $R_{mn}^x$ is asymptotically bounded as shown below:
\begin{equation}\label{eq:3rd}
    \begin{aligned}
        \mathbb{P}(\big|R_{mn}^x\big| \geq \varepsilon) & = \mathbb{P}\left(\Bigg| \frac{1}{m(m-1)}\sum_{i=1}^{m(m-1)}y_i - \frac{1}{n^2}\sum_{j=1}^{n^2}y_j\Bigg| \geq \varepsilon\right) \\
        & = \mathbb{P}\left(\Bigg| \frac{1}{m(m-1)}\sum_{i=1}^{m(m-1)}y_i - \mathbb{E}(y) + \mathbb{E}(y) - \frac{1}{n^2}\sum_{j=1}^{n^2}y_j\Bigg| \geq \varepsilon\right) \\
        & \leq \mathbb{P}\left(\Bigg| \frac{1}{m(m-1)}\sum_{i=1}^{m(m-1)}y_i - \mathbb{E}(y) \Bigg| \geq \frac{\varepsilon}{2}\right) + \mathbb{P}\left(\Bigg| \frac{1}{n^2}\sum_{j=1}^{n^2}y_j - \mathbb{E}(y) \Bigg| \geq \frac{\varepsilon}{2}\right) \\
        & \leq 2\exp\left(\frac{-2m(m-1)(\varepsilon/2)^2}{{K_+^y}^2}\right) + 2\exp\left(\frac{-2n^2(\varepsilon/2)^2}{{K_+^y}^2}\right) \\
        & \leq 4\exp\left(\frac{-m(m-1)\varepsilon^2}{2{K_+^y}^2}\right),
    \end{aligned}
\end{equation}
where $K_+^y \coloneqq \sup_{i,j} (y_i - y_j)$. Similarly, for $R_{mn}^\xi$ in equation \eqref{eq:MMDMMD}, we define $\theta \coloneqq {\bm{\delta}^b} \cdot \bm{\delta}^{\xi*} \in \mathbb{R}$ and $K_+^\theta \coloneqq \sup_{i,j} (\theta_i - \theta_j)$. Then, $R_{mn}^\xi$ is asymptotically bounded as follows:
\begin{equation}\label{eq:4th}
    \mathbb{P}(\big|R_{mn}^\xi\big| \geq \varepsilon) \leq 4\exp\left(\frac{-m^2\varepsilon^2}{2{K_+^\theta}^2}\right).
\end{equation}
By \Cref{le:1/nepsilon} and equations \eqref{eq:1st2nd}, \eqref{eq:3rd}, and \eqref{eq:4th}, we have:
\begin{equation}
    \begin{aligned}
        & \mathbb{P}\left(\big| MMD^2(\bm{\delta}_m^x, \bm{\delta}_n^\xi) - MMD^2(P_{\bm{\delta}^{x*}}, P_{\bm{\delta}^{\xi*}}) \big| \geq \varepsilon \right) \\
        & \leq \mathbb{P}\left(\big|MMD^2(\bm{\delta}_m^{x*}, \bm{\delta}_n^{\xi*}) - MMD^2(P_{\bm{\delta}^{x*}}, P_{\bm{\delta}^{\xi*}}) \big| \geq \frac{\varepsilon}{4} \right) \\
        & \quad + \mathbb{P}\left(\big|MMD^2(\bm{\delta}_m^{b}, \bm{\delta}_n^{b})\big| \geq \frac{\varepsilon}{4} \right) + \mathbb{P}(\big|2R_{mn}^x\big| \geq \frac{\varepsilon}{4}) + \mathbb{P}(\big|2R_{mn}^\xi\big| \geq \frac{\varepsilon}{4}) \\
        & \leq 4\exp\left(\frac{-Cn\varepsilon^2}{128(1+C)K_+^2}\right) + 4\exp\left(\frac{-m(m-1)\varepsilon^2}{128K_+^2}\right) + 4\exp\left(\frac{-m^2\varepsilon^2}{128K_+^2}\right),
    \end{aligned}
\end{equation}
where $K_+ \coloneqq \max\{K_+^x, K_+^b, K_+^\xi, K_+^\theta\}$. When $m > 1 + (1+C)^{-1}$, the following inequality always holds\footnote{In fact, this condition is always satisfied as long as there are more than only two normal samples.}:
\begin{equation}
    \frac{Cn}{1+C} < m(m-1) < m^2.
\end{equation}
Finally, we have:
\begin{equation}
    \mathbb{P}\left(\big| MMD^2(\bm{\delta}_m^x, \bm{\delta}_n^\xi) - MMD^2(P_{\bm{\delta}^{x*}}, P_{\bm{\delta}^{\xi*}}) \big| \geq \varepsilon \right) \leq 12\exp\left(\frac{-Cn\varepsilon^2}{128(1+C)K_+^2}\right).
\end{equation}
If $\alpha\coloneqq12, \beta\coloneqq(128K_+^2)^{-1}$, \Cref{Theorem 3.3} is proven.
\end{proof}

\setcounter{section}{1}
\subsection{Lemmas}
\begin{lemma}[$k$-dimensional Ramanujan's master theorem \cite{bradshaw2023operational,amdeberhan2012ramanujan}]\label{le:MultiRam}
    If a complex-valued function $f(x_1, \cdots, x_k)$ has an expansion:
    \begin{equation}
        f(x_1, \cdots, x_k) = \sum_{n_1, \cdots, n_k}^{\infty} g(n_1, \cdots, n_k) \prod_{i=1}^k \frac{(-1)^{n_i}}{n_i!} x_i^{n_i},
    \end{equation}
    where $g(n_1, \cdots, n_k)$ is a continuously analytic function everywhere, then the $k$-dimensional Mellin transform satisfies a multivariate version of Ramanujan's master theorem as follows:
    \begin{equation}
        \begin{aligned}
            \mathcal{M}[f(x_1, \cdots, x_k)](s_1, \cdots, s_k) & \coloneqq \int_{\mathbb{R}_+^k} \prod_{i=1}^{k} x_i^{s_i-1} f(x_1, \cdots, x_k) dx_1 \cdots dx_k \\
            & = \prod_{i=1}^{k} \Gamma(s_i)g(-s_1, \cdots, -s_k).
        \end{aligned}
    \end{equation}
    The integral is convergent when $0 < Re(s_i) <1, \forall i \in \{1,\cdots,k\}$.
\end{lemma}

\begin{lemma}[Asymptotic convergence of MMD \cite{gretton2012kernel}]\label{le:MMDcon}
    Given two sample sets $\bm{X}=\{\bm{x}_i|\bm{x}_i\sim p\}$, $\bm{Y}=\{\bm{y}_j|\bm{y}_j\sim q\}$, where $p$ and $q$ are the underlying distributions, respectively. Assume $0 \leq k(\bm{x}_i, \bm{x}_j) \leq K $. We have:
    \begin{equation}
        \mathbb{P}\left(\big|MMD^2(\bm{X}, \bm{Y}) - MMD^2(p, q)\big| \geq \varepsilon \right) \leq 2\exp\left(\frac{-\varepsilon^2mn}{8K^2(m+n)}\right),
    \end{equation}
\end{lemma}

\begin{lemma}\label{le:1/nepsilon}
    For any random variables $x_1, x_2, \cdots, x_k \in \mathbb{R}$, they always satisfy:
    \begin{equation}
        \mathbb{P}\left(\Bigg|\sum_{i=1}^{k}x_i\Bigg| \geq \varepsilon\right) \leq \mathbb{P}\left(\sum_{i=1}^{k}|x_i| \geq \varepsilon\right) \leq \sum_{i=1}^{k}\mathbb{P}\left(|x_i| \geq \frac{\varepsilon}{k}\right),
    \end{equation}
    where $\varepsilon \geq 0$, $k \in \mathbb{Z}_+$.
\end{lemma}

\begin{proof}
According to the triangle inequality, we have:
\begin{equation}
    \Bigg|\sum_{i=1}^{k}x_i\Bigg| \leq \sum_{i=1}^{k}|x_i|,
\end{equation}
which also indicates
\begin{equation}
    \left\{\Bigg|\sum_{i=1}^{k}x_i\Bigg| \geq \varepsilon\right\} \subset \left\{\sum_{i=1}^{k}|x_i| \geq \varepsilon\right\}.
\end{equation}
The following equation always holds:
\begin{equation}
    \left\{\sum_{i=1}^{k}|x_i| \geq \varepsilon\right\} \subset \bigcup_{i=1}^k \left\{|x_i| \geq \frac{\varepsilon}{k}\right\}.
\end{equation}
Thus, we have:
\begin{equation}
    \mathbb{P}\left(\Bigg|\sum_{i=1}^{k}x_i\Bigg| \geq \varepsilon\right) \leq \mathbb{P}\left(\sum_{i=1}^{k}|x_i| \geq \varepsilon\right) \leq \sum_{i=1}^{k}\mathbb{P}\left(|x_i| \geq \frac{\varepsilon}{k}\right).
\end{equation}

\end{proof}

\section{Automatically Infer Anomaly Cluster Numbers}\label{clustnum}
In case the true number of anomaly subtypes is unknown, ACSleuth can estimate this number based on a cell-cell similarity matrix calculated on fused anomalous cell embeddings (see Section 3.3 in the main text)\cite{Clunum}. Initially, give the fused anomalous cell embedding matrix $ \Xi \in \mathbb{R}^{N_{an} \times p}$, where $N_{an}$ denotes the anomaly cell numbers and $p$ the embedding dimensions, is subjected to row normalization followed by the calculation of a cell-cell similarity matrix $\mathbb{S}\in\mathbb{R}^{N_{an} \times N_{an}}$ as:
\begin{equation}
    \mathbb{S} = \frac{\Xi\Xi^T}{\max\left\{\Xi\Xi^T\right\}}.
\end{equation}
Next, $\mathbb{S}$ is transformed to a graph Laplacian matrix $\mathbb{L}\in\mathbb{R}^{N_{an}\times N_{an}}$ as:
\begin{equation}
    \mathbb{S}^\prime=\mathbb{S}+\mathbb{S}^2,
\end{equation}
\begin{equation}
    \mathbb{L}=\mathbb{D}^{-\frac{1}{2}}\mathbb{S}^\prime\mathbb{D}^{-\frac{1}{2}}.
\end{equation}
Here, $\mathbb{D}\in\mathbb{R}^{N_{an} \times N_{an}}$ represents the degree matrix of $\mathbb{S}$, and $\mathbb{S}^\prime$ aims to enhance the similarity structure. Then, the eigenvalues of $\mathbb{L}$ are ranked as $\lambda_{\left(1\right)}\le\lambda_{\left(2\right)}\le\cdots\le\lambda_{\left(N_{an}\right)}$. Finally, the number of subtypes can be inferred as:
\begin{equation}
    n_{inf}={\rm argmax}_i\left\{\lambda_{\left(i\right)}-\lambda_{\left(i-1\right)}\right\},\ \ \ \ i=2,\ 3,\ \cdots,\ N_{an}.
\end{equation}

\section{Datasets}\label{dataset}
\subsection{Datasets Detailed Information}
\begin{table}[H]
    \renewcommand{\arraystretch}{1.25}  
    \centering
    \setlength{\tabcolsep}{1.5pt}  
    \resizebox{\linewidth}{!}{
    \begin{tabular}{|c|c|c|c|c|}
    \hline
    \textbf{Data Source} & \textbf{Dataset Name} & \textbf{\#Instance} & \textbf{\makecell{Sequencing 
Technology}} &\textbf{Anomaly Ratio (\%)} \\
    \hline
    \multirow{2}{*}{TME \cite{TME}} & a-TME-$0$ & 3559 & \multirow{2}{*}{10x} & - \\
    & a-TME-$1$ & 2968 & & 60.24 \\
    \hline
    \multirow{3}{*}{Brain \cite{Brain}} & a-Brain-$0$ & 3203 & \multirow{3}{*}{Hi-Seq} & - \\
    & a-Brain-$1$ & 3448 & & 39.07 \\
    & a-Brain-$2$ & 3750 & & 7.31 \\
    \hline
    \multirow{3}{*}{PBMC \cite{PBMC}} & r-PBMC-$0$ & 4698 & \multirow{3}{*}{10x} & - \\
    & r-PBMC-$1$ & 3684 & & 13.14 \\
    & r-PBMC-$2$ & 3253 & & 12.73 \\
    \hline
    \multirow{3}{*}{Cancer \cite{Cancer}}  & r-Cancer-$0$ & 8104 & \multirow{3}{*}{10x} & - \\
    & r-Cancer-$1$ & 7721 & & 50.03 \\
    & r-Cancer-$2$ & 4950 & & 58.12 \\
    \hline
    \multirow{4}{*}{CSCC \cite{CSCC}} & r-CSCC-$0$ & 9242 & \multirow{4}{*}{10x} & - \\
    & r-CSCC-$1$ & 1409 & & 73.46 \\
    & r-CSCC-$2$ & 2681 & & 22.68 \\
    & r-CSCC-$3$ & 2091 & & 32.52 \\
    \hline
    \multirow{3}{*}{Pancreas \cite{Pancreas}}  & r-Pan-$0$ & 7010 & InDrop & - \\
    & r-Pan-$1$ & 3514 & Smart-Seq2 & 8.51 \\
    & r-Pan-$2$ & 3072 & CEL-Seq2 & 13.41 \\
    \hline
    \multirow{3}{*}{KDDRev \cite{KDD}} & KDDRev-$0$ & 7278 & \multirow{3}{*}{-} & - \\
    & KDDRev-$1$ & 5832 & & 12.00 \\
    & KDDRev-$2$ & 815 & & 87.48 \\
    \hline
    \end{tabular}
}
    \caption{Datasets detailed information. The dataset name is formatted as (data type)-(data name)-(dataset ID). For data type, a and r denote scATAC-seq and scRNA-seq, respectively. Dataset ID $0$ is reserved for reference datasets. \#Instance denotes the number of instances in each dataset, and the anomaly ratio is the proportion of true anomalies.}
    \label{tab:dataset}
\end{table}
\subsection{KDDCUP-Rev Dataset}
 The KDD-Rev dataset is derived from the KDDCUP99 10 percent dataset, which contains 34 continuous and 7 categorical attributes. Domain shifts exist between records of different connection protocol types. In our settings, the reference dataset only includes UDP records, while the two target datasets contain ICMP and TCP records. There are four types of cyber-intrusions: DOS, R2L, U2R, and PROBING. 



\end{document}